\documentclass[letterpaper, 10 pt, conference, nofonttune]{ieeeconf}
\IEEEoverridecommandlockouts
\usepackage{microtype}

\usepackage{microtype}
\usepackage{graphicx}
\usepackage{amssymb}
\usepackage{amsmath}
\usepackage[ruled,vlined,linesnumbered]{algorithm2e}
\usepackage{subfigure}
\usepackage{xcolor}
\usepackage{comment}

\usepackage{amsthm}

\newcommand{\Real}{\mathbb{R}}

\newcommand{\ppp}{\mathbf{p}}
\newcommand{\xxx}{\mathbf{x}}
\newcommand{\aaa}{\mathbf{a}}
\renewcommand{\ggg}{\mathbf{g}}
\newcommand{\zzz}{\mathbf{z}}
\newcommand{\ooo}{\mathbf{o}}
\newtheorem{theorem}{Theorem}

\newcommand{\vneg}{\vspace{-10pt}}

\title{\LARGE \bf Decentralized and Communication-Free Multi-Robot \\
Navigation through Distributed Games}

\author{
    Brian Reily$^1$, Terran Mott$^2$, and Hao Zhang$^2$%
    \thanks{$^1$: Brian Reily is affiliated with the DEVCOM Army Research
    Laboratory through ORAU. email: breily@mines.edu}
    \thanks{$^2$:Terran Mott and Hao Zhang are with 
    the Human-Centered Robotics Lab 
    in the Department of Computer Science
    at the Colorado School of Mines, Golden, CO, 80401.
    email: \{terranmott, hzhang\}@mines.edu.}%
}

\begin{document}
\maketitle
\thispagestyle{empty}
\pagestyle{empty}

\begin{abstract}

Effective multi-robot teams require the ability to move to goals in complex environments in order to address real-world applications such as search and rescue. Multi-robot teams should be able to operate in a completely decentralized manner, with individual robot team members being capable of acting without explicit communication between neighbors. In this paper, we propose a novel game theoretic model that enables decentralized and communication-free navigation to a goal position. Robots each play their own distributed game by estimating the behavior of their local teammates in order to identify behaviors that move them in the direction of the goal, while also avoiding obstacles and maintaining team cohesion without collisions. We prove theoretically that generated actions approach a Nash equilibrium, which also corresponds to an optimal strategy identified for each robot. We show through extensive simulations that our approach enables decentralized and communication-free navigation by a multi-robot system to a goal position, and is able to avoid obstacles and collisions, maintain connectivity, and respond robustly to sensor noise.

\end{abstract}

\section{Introduction}
\label{sec:intro}

Multi-robot teams are a viable solution for a variety of applications 
that involve distributed or simultaneous tasks, such as 
disaster response \cite{erdelj2017wireless},
area exploration \cite{bayat2017environmental}, 
and search and rescue \cite{goldhoorn2018searching}.
To effectively address such applications, robots must be able to move
through complex environments 
and arrive at waypoints and goal positions \cite{manderson2020vision}.
While planning its own navigation actions, each robot must 
avoid obstacles as well as collisions with its teammates who are 
making their own movements.
Accordingly, each robot needs the capability to anticipate both
the actions of their local teammates and 
how their own actions will affect their neighbors.

In order to ensure the resilience of multi-robot teams,
robots must be able to work in a decentralized manner in the real world
\cite{smyrnakis2014coordination,banfi2018optimal}.
Complex environments can be dangerous, and relying on a single robot
to provide centralized control can cause an entire multi-robot system
to fail due to the failure of only that single robot \cite{luo2019minimum}.
Decentralized systems are resilient to this, continuing to operate
despite failure of individual robots \cite{li2020resilient}.
In addition, in a resilient multi-robot system robots should be able
to operate without requiring the capability to communicate with their
teammates. 
Multi-robot teams are often composed of low-cost robots that may 
lack expensive communication capabilities, or multi-robot teams may
lose members due to obstacles and adversaries that interfere with
communication.
Together, enabling decentralized behaviors
and allowing robots to operate without communication to their 
neighbors is vital for the success of the multi-robot team.

\begin{figure}[t]
    \vneg
    \centering
    \includegraphics[width=0.41\textwidth]{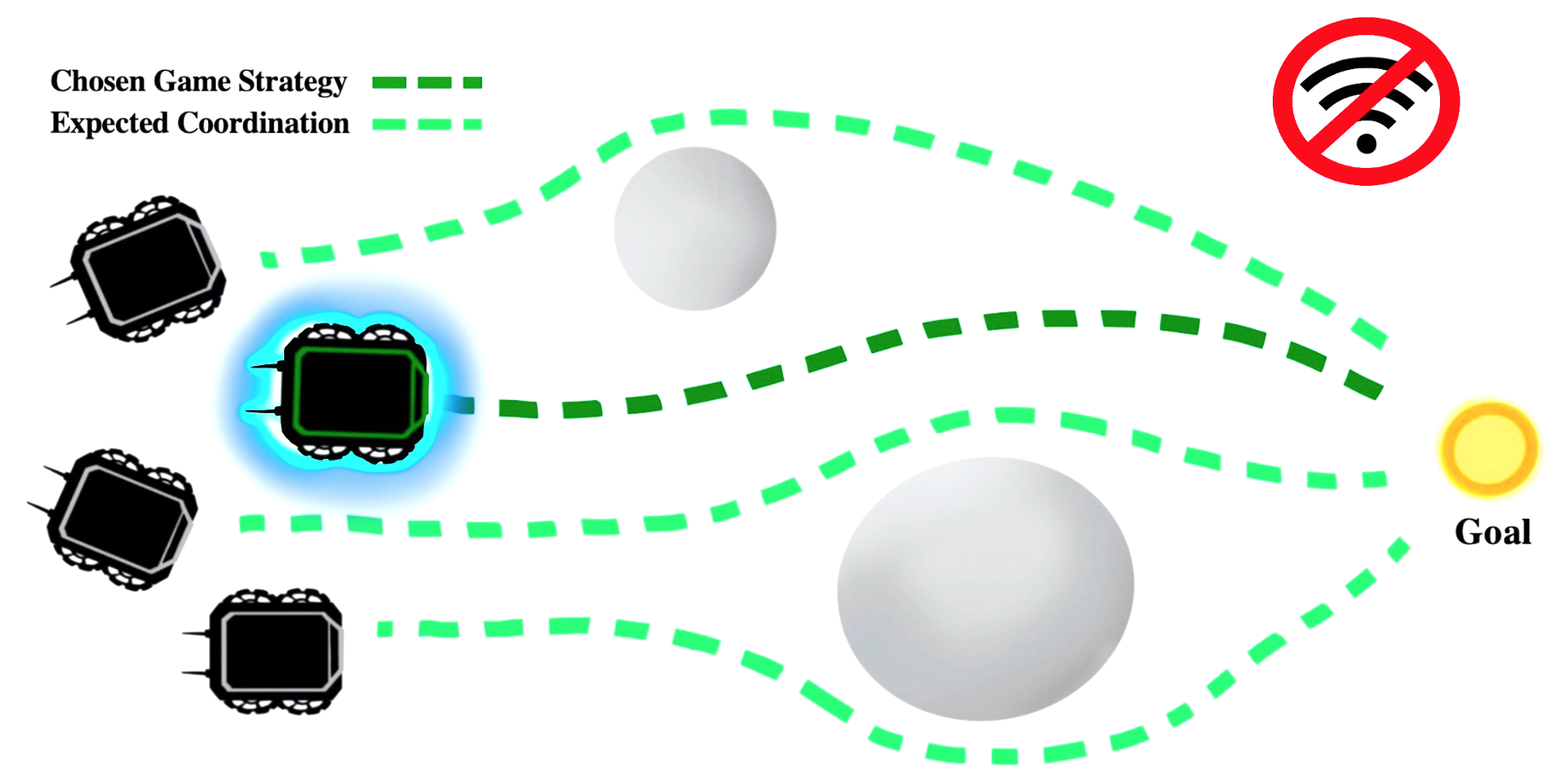}
    \caption{Our approach enables decentralized and communication-free
    navigation.
    The robot highlighted in blue chooses
    its strategy after considering the expected
    strategies of its neighbors and the effect of its strategy
    on them, without actually communicating.
    In this way, the blue robot can
    avoid collisions, maintain team cohesion, avoid obstacles,
    and make progress towards a goal position without direct communication.
    }
    \label{fig:motivation}
\end{figure}

Due to its importance, multi-robot navigation has been explored by 
a wide variety of methods.
Many methods have been designed based on centralized planning and control, 
which require a master robot for centralized computing and communication 
among all teammates 
\cite{aranda2019deformation,reily2020leading,varava2017herding}.
Multiple other methods have been developed to work in a decentralized manner,
but still rely on communication among robots 
\cite{lin2019connectivity, tolstaya2019learning}.
The problem of navigating multiple robots towards a goal
in a completely decentralized and communication-free manner has not
been well addressed, mainly because without communication it is challenging
for a robot to anticipate actions of neighboring robots that dynamically 
make independent decisions \cite{emery2005game_a, emery2005game_b}.
Recently, game theory has attracted increasing attention 
in coordinating robots without communication \cite{bansal2020bayesian}.
Game theoretic methods have been developed to identify 
actions in the context of their effect on teammates and adversaries for
two-player applications such as driving
\cite{ding2018game,liniger2020safe}, 
or racing \cite{wang2019game_a}, as well as larger games in
known environments for surveillance \cite{hernandez2014selective}.
However, they have not been previously utilized to address the problem of 
navigating a variable size multi-robot system through a complex
environment to a goal position.

In this paper, we propose a novel game theoretic and
communication-free approach to multi-robot navigation.
Our approach assumes that a robot is able to locally observe locations
of nearby obstacles and robots in the multi-robot system, but does not 
require communication with teammates or knowledge of the entire
system.
Then, our approach models each robot and its neighboring teammates with a
strategic game, formulating multi-robot navigation as a problem
of solving a game for each robot at each time step.
These games are solved in a distributed fashion, with each robot playing 
its own individual game.
In each game, illustrated in Figure \ref{fig:motivation}, 
a robot estimates the behaviors of its 
neighbors in order to identify movements that avoid collisions,
while also progressing towards stationary or dynamic goal positions and 
avoiding obstacles.
Collision-free actions are determined by the robot to minimize
a cost function, which is theoretically proved to approach 
a Nash equilibrium, i.e., a state in the game where no player can
unilaterally improve their outcome.

The contribution of this paper is twofold.
\begin{enumerate}
    \item We propose a novel formulation that models multi-robot 
    navigation as multiple distributed normal-form games, allowing individual 
    robots to form precise expectations about the actions of their 
    neighbors and the effects of their own actions, which then enables communication-free and decentralized multi-robot navigation.
    \item We theoretically prove that these games approach stable 
    Nash equilibria, which also enables each robot in the team to identify 
    its optimal strategy without the requirement of communication with 
    other teammates.
\end{enumerate}

\section{Related Work}
\label{sec:related}

\subsection{Multi-Robot Maneuver}

Many multi-robot navigation methods require explicit communication
among robots, either to
avoid collisions with obstacles by generating a consensus behavior
among robots \cite{alonso2016distributed} or to calculate a global
gradient that controls the overall shape of the multi-robot formation
\cite{aranda2019deformation}.
Several methods have used leader and follower models to navigate a 
multi-robot system.
By controlling the movements of a number of robots, the overall 
behavior of the system was influenced \cite{sabattini2016coordinated}.
A single leader was utilized in \cite{kwon2019influencing}, basing
its movements on influencing nearby naive agents to follow it to goal
positions.
A single leader was also utilized in \cite{vrohidis2017safe} and multiple
leaders were used in \cite{reily2020leading}, but
both required guaranteed connectivity with the
leader.
A consistent drawback of these models is
a required communication, either for centralized control or for
knowledge of follower positions in the environment.

Recently, several decentralized methods have been implemented.
Graph neural networks have enabled decentralized behavior policies
\cite{tolstaya2019learning,khan2019graph}, but relied on 
local communication that
allowed information to propagate globally.
Graph temporal logic was applied to swarm guidance in order to
create policies that scale to large numbers of agents
\cite{djeumou2020probabilistic}, but relies on knowledge of
the overall agent distribution.
Deep reinforcement learning was used to learn navigation policies
around complex obstacles, allowing decentralized policies but 
requiring full knowledge of the state of the multi-robot system for training
\cite{lin2019end,lin2019connectivity}.

\subsection{Game Theory for Robotics}

Game theory seeks to explain and analyze the interaction of rational agents accommodating their own preferences and expectations. Modelling systems as games enables agents to 
analyze actions based on how they impact other agents.
Many methods have been designed to find solutions to games, 
including optimization and reinforcement learning
\cite{mertikopoulos2016learning,basar2018nonzero,halpern2017characterizing}.
Game theory has been applied in several multi-robot system applications \cite{robin2016multi},
due to its ability to consider the interaction of multiple
agents and their competing motivations.

In autonomous driving, game theory has been applied to the
problem of lane changing and merging \cite{liniger2020safe}.
Deep reinforcement learning was used to solve two player merging games
in \cite{ding2018game}, which relied on defining drivers as proactive or passive.
Roles for two players were also used in \cite{fisac2019hierarchical},
which modelled lane changing as a Stackleberg game that has implicit 
leaders and followers.
Multiple drivers were modelled in \cite{cleac2020algames},
which solved games through optimization by converting constraints such
as collision avoidance to Lagrangian multiplier terms.
In addition to lane changing, game theory was also used to generate racing
trajectories which would avoid collisions while still competing
with other agents \cite{wang2019game_a,wang2019game_b}.

Despite the exploration of game theory in autonomous driving,
previous research has mostly been limited to
only two player systems (i.e., two cars).
Recently, game theory has been applied to several multi-robot scenarios.
Many approaches have considered task allocation, both with the
requirement of communication \cite{roldan2018should} and enabling
this without communication being necessary
\cite{kanakia2016modeling,gao2019multirobot,dai2019task}.
Collision avoidance has been considered, but has required
the agents to be capable of communicating with each other
\cite{mylvaganam2017differential}.
Multi-robot sensing and patrolling has been the other main
scenario explored with game theoretic models, but approaches
have required prior mappings of the environment
\cite{meng2008multi,hernandez2014selective}
or still require communication between agents \cite{ramaswamy2016mutual}.

Game theoretic models have been used due to their useful
ability to model other agents in a system, but have not been 
theoretically well studied in generating navigational behaviors for 
a large multi-robot system without the ability to communicate between 
robots in complex environments.

\section{Our Proposed Approach}
\label{sec:approach}

\subsection{Multi-Robot Navigation as a Decentralized Game}
\label{sec:formulation}

We consider a multi-robot system consisting of $N$ robots.
Each robot has a position $\ppp \in \Real^d$, where $\ppp_i$ is the position of the $i$-th robot.
At each time step $t$, each robot makes a decision to take an action $\aaa_i^t$ and update its position.
We define a goal $\ggg \in \Real^d$ that is 
the position that the multi-robot system is attempting to 
navigate to, which is known to each robot.
This goal can be stationary (e.g., a fixed location in the environment) 
or dynamic (e.g., a moving human whom the multi-robot system must follow).
We define the multi-robot system to have reached the goal state
when $\| \ppp_+ - \ggg \|_2 < T$, where $\ppp_+$ is the geometric center
of the robot system and $T$ is a distance
threshold.

We define an ideal interaction distance $R$ between robots,
where if $\| \ppp_i - \ppp_j \|_2 \ll R$ then
the $i$-th and $j$-th robots are at risk of collision (with the
collision risk rising as $\ppp_i$ approaches $\ppp_j$) and if
$\| \ppp_i - \ppp_j \|_2 \gg R$ then there
is a risk of losing cohesion between
robots $i$ and $j$ (with this risk rising as $\ppp_i$ moves
further away from $\ppp_j$).
Successful cohesion in the multi-robot system results from maintaining
$\| \ppp_i - \ppp_j \|_2 = R \; \forall \: i, j \in \{ 1, \dots, N \}, i \neq j$.
Finally, we consider a set of $O$ obstacles in the environment.
Robots should take actions such that $\| \ppp_i - \ooo_j \|_2 > 0$,
so as to avoid collisions with obstacles.
We assume each robot is able to locally perceive neighboring 
robots and obstacles.

Then, we formulate the problem of multi-robot navigation with multiple
distributed and 
differentiable simultaneous strategic games
$\mathcal{G} = \{ \mathcal{N}, \mathcal{X}, \mathcal{J} \}$,
where each robot is the main player of a game based on its local environment.
Each game is \emph{simultaneous} as the robot plays its game at each
time step without knowledge of how other robots will move.
Each game is also \emph{strategic}, where players are unable to 
coordinate through explicit communication.
$\mathcal{N}$ represents the set of players, where the 
cardinality of $\mathcal{N}$ is denoted as $N'$.
This set of players is defined as the main player robot and
the neighboring robots within sensing distance.
$\mathcal{X} = \{ X_1, \dots, X_{N'} \}$ represents the strategy set for
game $\mathcal{G}$, where $X_i$ represents the set of available strategies of the $i$-th
player.
We define a strategy $\xxx_i \in X_i$ to be a movement update
available to the $i$-th robot given $\ppp_i$ and actions $\aaa_i$
\begin{equation}
\xxx_i = \ppp_i + \alpha  \aaa_i
\vspace{-5pt}
\label{eq:strategy}
\end{equation}
where $\aaa_i$ satisfies $\| \aaa_i \|_2 = 1$ and $\alpha$ is a hyperparameter
controlling the importance of the action.
Values of $\alpha$ are discussed later in the paper when 
consider the theoretical aspects of the games.

We define 
$\mathcal{J} = \{ J_1 ( \xxx_1, \xxx_{-1} ), \dots, J_{N'} ( \xxx_{N'}, \xxx_{-N'} ) \}$ 
as the set of \emph{differentiable} cost functions that govern the 
action of each player, where $\xxx_i$ is the strategy of the $i$-th 
player and $\xxx_{-i}$ is the strategy set of all other players, i.e.
$\xxx_{-i} = ( \xxx_1, \dots, \xxx_{i-1}, \xxx_{i+1}, \dots, \xxx_{N'} )$.
Each player is assumed to be \emph{rational},
i.e.,  each player will select a strategy $\xxx_i$ that minimizes
its cost function $J_i ( \xxx_i, \xxx_{-i} )$, while expecting that 
every other player is selecting strategies in $\xxx_{-i}$ that minimize
their own cost functions.
We denote the optimal strategy for player $i$ as $\xxx_i^*$ and the optimal
strategies for the other players as $\xxx_{-i}^*$.
The formulated cost function common to each player is defined as:
\begin{equation}
\small{
J_i ( \xxx_i, \xxx_{-i} ) = \beta_1 \| \xxx_i - \ggg \|_2^2 + \beta_2 C_i( \xxx_i, \xxx_{-i} ) - \beta_3 \sum_{o=0}^O \| \xxx_i - \ooo_o \|_2^2
}
\label{eq:cost}
\end{equation}
where $C_i ( \xxx_i, \xxx_{-i} )$ represents a cost related to 
cohesion with other players within sensing distance, which can be defined as 
$C_i( \xxx_i, \xxx_{-i} ) = \sum_{n=1, n \neq i}^{N'} \| (\xxx_i - \zzz_n) \|_2^2$,
where $\zzz_n$ defines a point that is on the axis between a
sensed robot $n$ and robot $i$ and is distance $R$ from robot $n$
(e.g., if the distance between robots is less than $R$, this cost
will encourage robot $i$ to move away from robot $n$).
This is calculated from the position of robot $i$
using the estimated strategy $\xxx_n$ of the neighboring robot $n$,
such that $\zzz_n = R ( \ppp_i - \xxx_n ) / ( \| \ppp_i - \xxx_n \|_2 )$.

Hyperparameters $\beta_1, \beta_2, \beta_3$
balance the different influences and are valued
such that $\beta_1 \geq O \beta_3 - N' \beta_2$.
$\beta_1$ controls the weighting of progress towards the
goal.
This part of the cost function decreases as the $i$-th robot
makes progress towards the goal.
$\beta_2$ controls the weighting of cohesion with 
neighboring robots.
This cohesion cost is minimized when the $i$-th robot is at a
distance $R$ from every neighbor.
$\beta_3$ controls the weighting of navigating to avoid obstacles.
As robots move away from obstacles, this term minimizes the 
overall cost function.

\subsection{Solution and Analysis of the Game}
\label{sec:solution}

In order to identify optimal behaviors for each robot, 
our approach constructs the described game $\mathcal{G}_i$ for each robot 
$i \in \{ 1, \dots, N \}$ at each time step.
Algorithm \ref{alg:solve} solves the game required to enable decentralized
and communication-free multi-robot navigation.
The algorithm identifies the players involved in each game $\mathcal{G}_i$,
consisting of only the local neighbors of robot $i$,
and initializes a null strategy for the main player.
Then, it minimizes each neighboring player's cost function to
identify optimal strategies $\xxx_{-i}^*$.
From this, it updates the main player's cost function and identifies the
optimal strategy $\xxx_i^*$ that minimizes it.
When Algorithm \ref{alg:solve} converges after iteration, 
$\xxx_i^*$ is adopted as the strategy for robot $i$.
In empirical evaluation, Algorithm \ref{alg:solve} converges
quickly, typically in less than $N'$ iterations.

For each game, we have two objectives for the main player 
(denoted as robot $i$).
First, our approach should minimize the cost function 
$J_i ( \xxx_i, \xxx_{-i} )$
described in Eq. (\ref{eq:cost}), in order to identify the best strategy 
for this main player.
Second, our approach should generate strategies that move in the direction
of a Nash equilibrium for the game, or a state where no player can improve their
strategy without other players changing their strategies.
If every player acts rationally, as we are assuming, then a Nash equilibrium
means that robot $i$ has identified a strategy $\xxx_i$ that cannot be
improved upon, assuming its expectations of its neighbors' optimal strategies $\xxx_{-i}^*$ are accurate.

\subsubsection{Minimizing the Cost Function}

First, we consider the objective of minimizing the cost function 
$J_i ( \xxx_i, \xxx_{-i} )$
for the main player of the game $\mathcal{G}_i$.
We have defined our game $\mathcal{G}$ and our cost functions $\mathcal{J}$
as \emph{differentiable}, and so we are able to utilize derivative-based
methods in order to identify minima which correspond to equilibria in
the game.
To take advantage of this fact, we define an action $\aaa_i$ that is part
of a strategy $\xxx_i$ with the gradient of the cost function:
\begin{align}
\label{eq:actions}
\aaa_i  & = - \nabla_i J_i ( \xxx_i, \xxx_{-i} )
\vspace{-7pt}
\end{align}
with the exact gradient of Eq. (\ref{eq:cost}) omitted for space.
Just as gradient descent moves in the direction that minimizes a function,
by defining our actions this way robots will take physical moves in
the direction of minimizing their cost functions.
Then, $\aaa_i$ is normalized to be a unit vector by 
$\aaa_i / \| \aaa_i \|_2$ in order to represent real-world 
restrictions on robot movements at a single time step.

\begin{algorithm}[h]
\vspace{-1pt}
\SetAlgoLined
\SetKwInOut{Input}{Input}
\SetKwInOut{Output}{Output}
\SetNlSty{textrm}{}{:}
\SetKwComment{tcc}{/*}{*/}
\KwIn{A multi-robot system with $N$ robots.}
\KwOut{Optimal strategies $\xxx_i^*$ for each robot.}

\For{$i \gets 1$ \KwTo $N$}{
Identify positions of neighboring robots and obstacles within sensing distance of robot $i$.

Construct game $\mathcal{G}_i = \{ \mathcal{N}, \mathcal{X}, \mathcal{J} \}$.

Initialize strategy $\xxx_i$ with a zero action $\aaa_i$.

\Repeat{
convergence
}{
Estimate strategies $\xxx_{-i}^*$ for neighbors.

Update strategy $\xxx_i$ to minimize $J_i ( \xxx_i, \xxx_{-i}^* )$.
}
}
\caption{Solving the Decentralized Game.}
\label{alg:solve}
\vspace{-3pt}
\end{algorithm}

We now introduce constraints on $\mathcal{X}$ to include explicit
collision avoidance, both with obstacles and neighboring robots.
These constraints ensure that the valid set of strategies 
$\mathcal{X}_i$ for robot $i$ does not allow actions that would 
move it within the estimated position of a neighbor or obstacle:
\begin{align}
\label{eq:a_constraints}
X_i = & \Big\{ \xxx_i : \xxx_i = \ppp_i + \alpha \aaa_i, \| \aaa_i \|_2 = 1,
\| \xxx_i - \ppp_j \|_2 \geq \tau, \\ 
& \;\; \| \xxx_i - \ooo_k \|_2 \geq \tau \; \forall j \in \mathcal{N}, i \neq j, k \in \{ 1, \dots, O \} \Big\} \notag
\vspace{-7pt}
\end{align}
where $\tau$ is a distance defined as that where collision occurs and
is dependent on the physical dimensions of the robots.
Though we propose a general solution where valid movements can 
occur in any direction (as with a holonomic robot), additional
constraints can be applied here in correspondence with a robot's 
physical and kinematic limitations.

Given Eqs. (\ref{eq:actions}) and (\ref{eq:a_constraints}), 
we prove in Theorem \ref{th:one} that our approach minimizes
the cost function of each robot.

\begin{theorem}
\label{th:one}
The inner loop of Algorithm \ref{alg:solve} minimizes the value of
$J_i ( \xxx_i, \xxx_{-i} )$ at each iteration for the main player $i$.
\end{theorem}
\begin{proof}
Eq. (\ref{eq:strategy}) computes a strategy $\xxx_i$ for the 
$i$-th player by $\xxx_i = \ppp_i + \alpha \aaa_i$.
From Eq. (\ref{eq:strategy}), we can also define the  null strategy 
$\xxx_i^0$ as that $\ppp_i$ with an action $\aaa_i = \mathbf{0}$ 
is equal to $\xxx_i$, meaning that $J_i (\xxx_i^0, \xxx_{-i} )$ 
returns the cost at the current position of robot $i$.
Finally, Eq. (\ref{eq:actions}) computes 
$\aaa_i = - \nabla_i J_i ( \xxx_i, \xxx_{-i} )$.
Combining these terms together results in:
\begin{align}
J_i ( \xxx_i, \xxx_{-i} ) = J_i ( \xxx_i^0, \xxx_{-i} ) - \alpha \nabla_i J_i ( \xxx_i, \xxx_{-i} )
\vspace{-7pt}
\label{eq:desc}
\end{align}
Thus, for sufficiently small step sizes of $\alpha$, 
Eq. (\ref{eq:desc}) is equivalent to gradient descent,
and thus decreases the cost function:
\begin{align}
J_i ( \xxx_i^{t+1}, \xxx_{-i}^{t+1} ) \leq J_i ( \xxx_i^{t}, \xxx_{-i}^t )
\vspace{-7pt}
\label{eq:reduction}
\end{align}
Further proof that gradient descent decreases function values is omitted for brevity. 
\end{proof}

\subsubsection{Approaching a Nash Equilibrium}

We now prove the objective of generating strategies for all players
$i \in \mathcal{N}$ such that each game moves in the direction of a Nash
equilibrium.
A Nash equilibrium results when no single agent can improve
its position without other agents changing their strategies.
In our case, when all agents are \emph{rational} in being
motivated to minimize their defined cost function, the Nash equilibrium
is also the optimal solution to a game.
Formally, a Nash equilibrium is achieved when an optimal strategy $\xxx_i^*$
is identified such that:
\begin{align}
J_i ( \xxx_i^*, \xxx_{-i}^* ) \leq J_i ( \xxx_i, \xxx_{-i}^* ) \quad \forall \xxx_i \in \mathcal{X}_i, \; i \in \mathcal{N}
\vspace{-7pt}
\end{align}

Mathematically, a Nash equilibrium occurs when the first derivative of
$J_i ( \xxx_i, \xxx_{-i} )  \in \mathcal{J}$ equals $\mathbf{0}$:
\begin{align}
\nabla_i J_i ( \xxx_i, \xxx_{-i} ) = \mathbf{0} \quad \forall i \in \mathcal{N}
\vspace{-7pt}
\label{eq:equil}
\end{align}
The set of first derivatives of all cost functions is defined as the {game gradient}:
\begin{align}
\nabla \mathcal{J} = [ ( \nabla_1 J_1 ( \xxx_i, \xxx_{-i} ) )^\top, \dots, ( \nabla_{N'} J_{N'} ( \xxx_i, \xxx_{-i} )  )^\top ]^\top \notag
\vspace{-7pt}
\end{align}
It is necessary for a Nash equilibrium that the game gradient 
$\nabla \mathcal{J} = \mathbf{0}$.
Additionally, a \emph{stable} Nash equilibrium occurs when the second derivative
of each cost function is positive;
that is, when the Jacobian of the cost functions is positive, or:
\begin{align}
\nabla_i^2 J_i ( \xxx_i, \xxx_{-i} ) \geq \mathbf{0} \quad \forall i \in \mathcal{N}
\vspace{-7pt}
\end{align}

In light of the above rigorous conditions, and given that all $J_i ( \xxx_i, \xxx_{-i} )$
defined by Eq. (\ref{eq:cost}) are parameterized by the cascading 
strategies of neighboring robots (i.e., when a neighboring robot moves,
it effects the cost functions of all nearby robots), 
a Nash equilibrium may not be theoretically possible to achieve in an
individual game (i.e., since a game is played at each time step and 
robots can only move so far physically at a time step).
Instead, we prove that strategies identified by Algorithm \ref{alg:solve}
approach a stable Nash equilibrium.
\begin{theorem}
Algorithm \ref{alg:solve} decides strategies for each player 
$i \in \mathcal{N}$ such that $\mathcal{G}_n$ approaches a stable 
Nash equilibrium.
\label{th:two}
\end{theorem}
\begin{proof}
In Eq. (\ref{eq:actions}), we define the action $\aaa_i$
taken using the gradient of the cost function.
From Theorem \ref{th:one}, we have proven that defining our action
in such a way minimizes the cost function and thus
approaches $\nabla_i J_i (\xxx_i, \xxx_{-i}) = \mathbf{0}$.
By taking this action (or an action in the direction of 
it, after normalizing to unit length), the $i$-th robot moves in the
direction of minimizing this cost function.
As Theorem \ref{th:one} holds for all $i \in \mathcal{N}$, 
the result of applying Algorithm \ref{alg:solve} for each robot
is that:
\begin{align}
\nabla_i J_i ( \xxx_i^t, \xxx_{-i}^t ) \to \mathbf{0} \; \text{as} \; t \to \infty \quad \forall i \in \mathcal{N}
\vspace{-7pt}
\end{align}
Thus, our approach results in each game approaching a Nash equilibrium.

\begin{figure}[]
    \centering
    \subfigure[Obstacle Field]{
        \label{fig:scenario_A}
        \centering
        \includegraphics[height=1.05in]{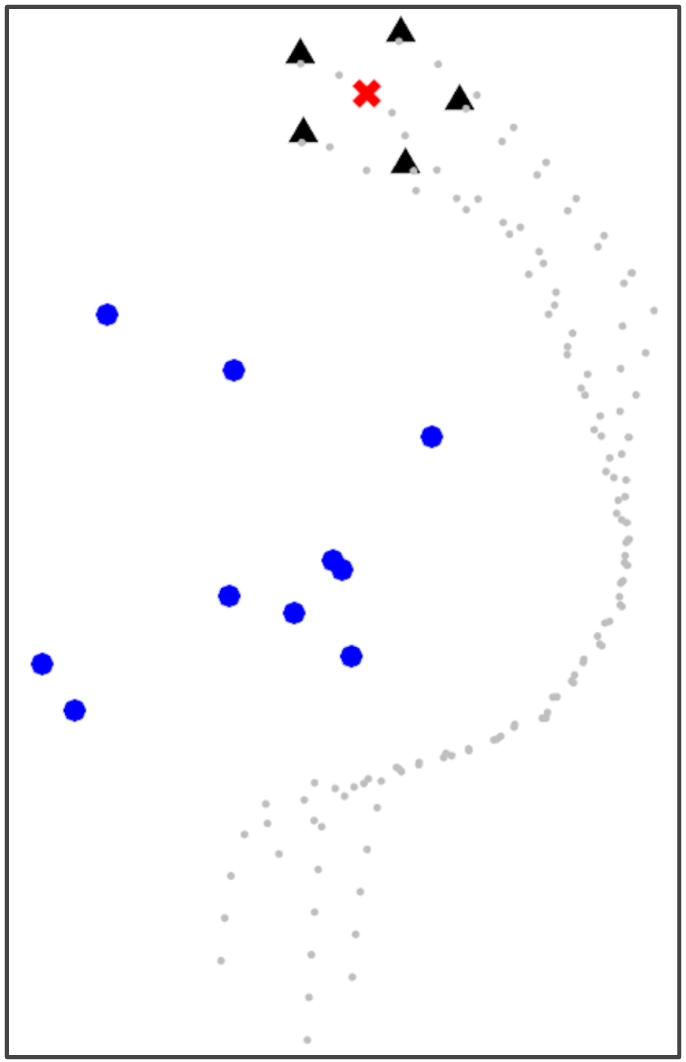}
    }%
    \subfigure[Area Obstacles]{
        \label{fig:scenario_B}
        \centering
        \includegraphics[height=1.05in]{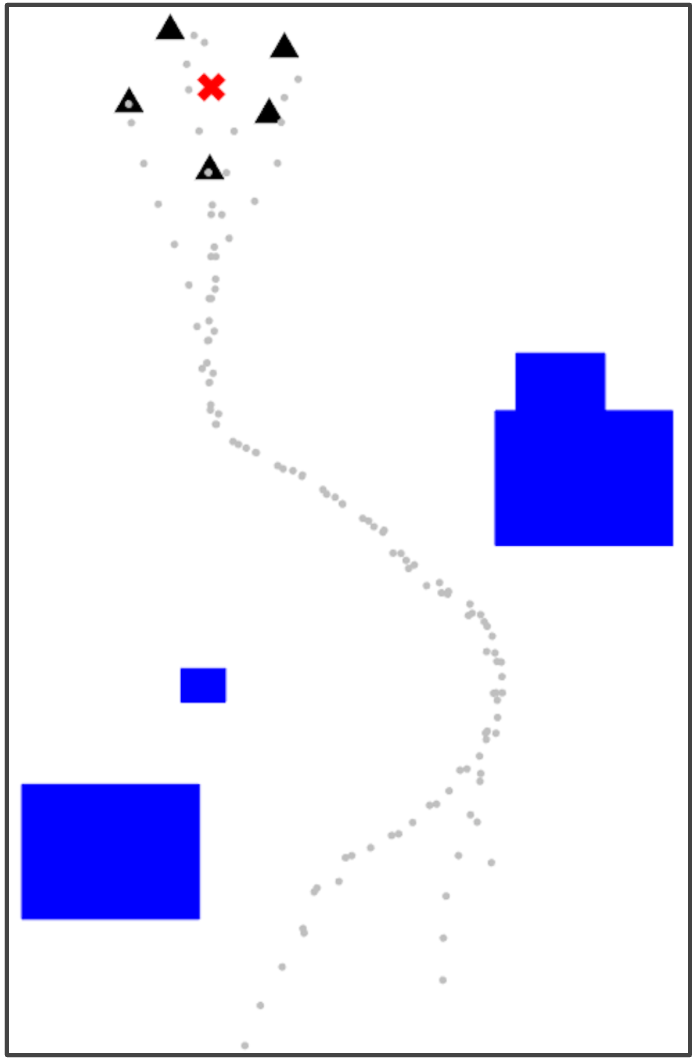}
    }%
    \subfigure[Dynamic Goal]{
        \label{fig:scenario_C}
        \centering
        \includegraphics[height=1.05in]{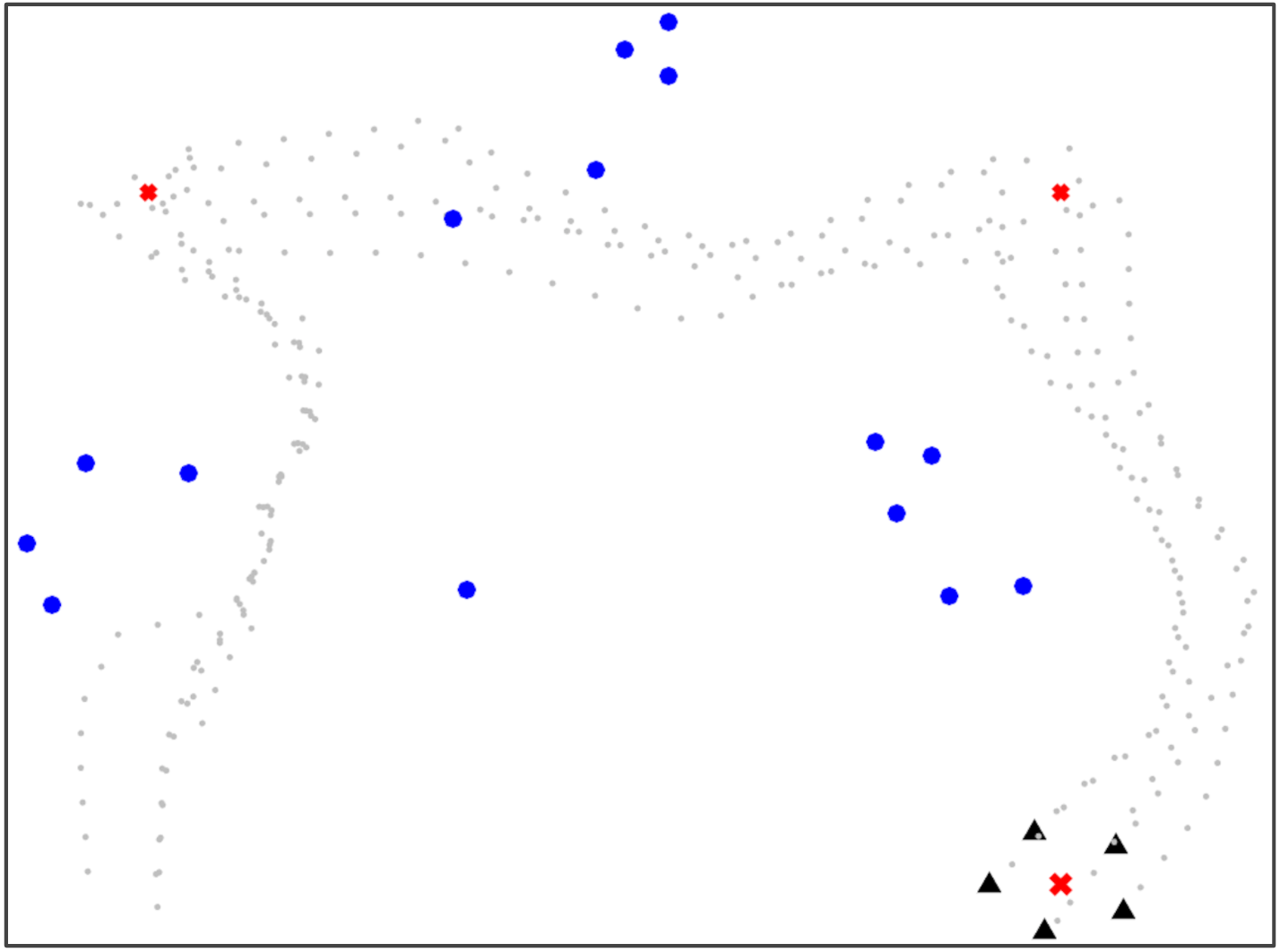}
    }
    \vneg
    \caption{Qualitative results showing the navigation of a 
    multi-robot system through an instance of each scenario.
    Robots are represented by black triangles, with their paths shown in
    light grey.
    The goal positions are marked by red `x' marks, with 
    obstacles marked in blue.
    }\label{fig:scenarios}
\end{figure}

Additionally, we earlier define that the hyperparameters of the
cost function must be chosen such that
$\beta_1 \geq O \beta_3 - N' \beta_2$.
In order for a Nash equilibrium to be \emph{stable}, the Jacobian 
of the cost functions is positive, requiring
$\nabla_i^2 J_i ( \xxx_i, \xxx_{-i} ) \geq \mathbf{0} \quad \forall i \in \mathcal{N}$.
The second derivative of our cost function is:
\begin{align}
\nabla_i^2 J_i = 2 \beta_1 + 2 N' \beta_2 - 2 O \beta_3
\vspace{-7pt}
\end{align}
Therefore, when the above condition for setting the hyperparameter values
is met, the Jacobian of the cost function is positive.
Thus, each robot $i$ in the game $\mathcal{G}_n$ moves in the
direction of a stable Nash equilibrium, which 
corresponds to an optimal movement for each robot. 
\end{proof}

Theorems \ref{th:one} and \ref{th:two} prove that our approach minimizes
the cost function for each robot in each game and
that every game moves robots in the direction of a stable Nash 
equilibrium, or an optimal strategy for each robot.
Thus, our approach enables optimal multi-robot navigation
with our defined cost function which moves towards the goal,
maintains ideal system cohesion, and avoids collisions with
obstacles.

\section{Experimental Results}
\label{sec:result}

We perform extensive evaluation of our approach,
comparing it to two baseline versions (setting either
$\beta_2$ or $\beta_3$ equal to 0), a greedy approach that 
moves towards the goal without consideration of obstacles or
collisions, and a random approach to provide a worst case 
comparison (as widely used in other works
\cite{kwon2019influencing,reily2020leading}).
Both the greedy and random approaches are restricted to be collision-free,
similar to the strategy constraints in Eq. (\ref{eq:a_constraints}).
We also evaluate with added sensor
noise to make robots less certain of the locations of their
neighbors and obstacles.
Each scenario is simulated 100 times to gain an understanding
of performance with a variety of starting
formations, obstacle arrangements, and goal locations.

We evaluate and compare our approach using four metrics. 
First, \emph{Cohesion} quantifies the
average distance between all robots in the system.
Ideally, the cohesion distance of a system is equal to 
$R$, the defined interaction distance that avoids
collisions and preserves cohesion.
Second, \emph{Connectivity} quantifies
the connectivity of the system.
We define a binary adjacency matrix where the $ij$-th element
equals $1$ if the $i$-th and $j$-th robot are within sensing
distance of each other, and $0$ otherwise.
We then calculate the 2nd smallest eigenvalue of the Laplacian
of this, known as the algebraic 
connectivity of the graph \cite{fiedler1973algebraic}.
We present this metric as a percentage, where $0$ indicates the
system is not connected and $100$ indicates the system is fully
connected.
Though our approach does not require communication,
connectivity has other benefits, such as relaying information about
malfunctions or events.
Third, the \emph{Nearest Obstacle} reports the average
nearest distance to an obstacle at each iteration.
Fourth, the \emph{Distance to Goal} shows the progress of
each compared approach towards the goal position, showing
$\| \ppp_+ - \ggg \|_2$ at each time step.

\subsection{Evaluation in Obstacle Field}

We first evaluated the Obstacle Field scenario seen in 
Figure \ref{fig:scenario_A}, with quantitative results seen in
Figure \ref{fig:scenario_A_results}.
We see that our full approach completes simulations with the closest
to the ideal interaction distance in Figure
\ref{fig:scenario_A1}, while maintaining it well
as time progresses.
The $\beta_3 = 0$ approach performs similarly well, as this still
considers system cohesion.
The greedy approach and $\beta_2 = 0$ approach, which do not value
cohesion, result in robots grouped much too close
together.
Figure \ref{fig:scenario_A2} displays the connectivity status
achieved by different methods over the course of the simulation.
Both our full approach and our $\beta_3 = 0$ approach maintain
connectivity well, as does the greedy approach, which
indirect causes robots to move towards one another at the goal.
Our $\beta_2 = 0$ approach is less able to maintain a
connected system as it moves through obstacles 
due to it not valuing spacing between robots, but increases 
connectivity as the system reaches the goal position.

\begin{figure*}[h]
    \centering
    \subfigure[Cohesion]{
        \hspace{-12pt}
        \label{fig:scenario_A1}
        \centering
        \includegraphics[width=0.215\textwidth]{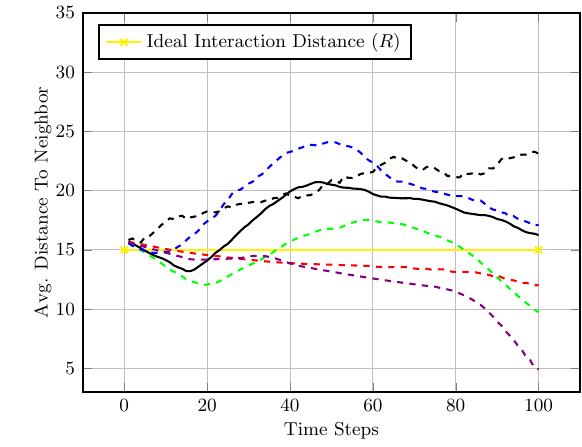}
    }%
    \subfigure[Connectivity]{
        \hspace{-5pt}
        \label{fig:scenario_A2}
        \centering
        \includegraphics[width=0.22\textwidth]{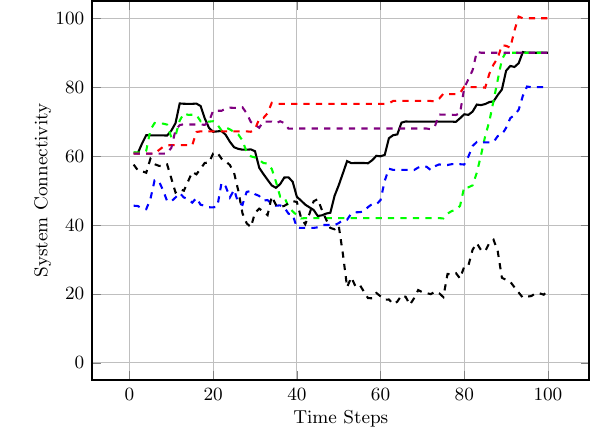}
    }%
    \subfigure[Nearest Obstacle]{
        \hspace{-5pt}
        \label{fig:scenario_A3}
        \centering
        \includegraphics[width=0.215\textwidth]{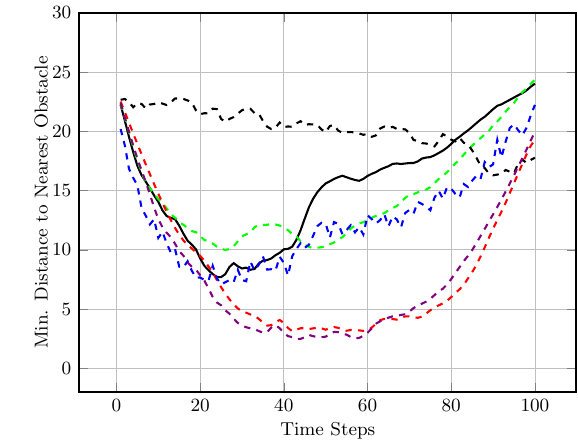}
    }%
    \subfigure[Distance to Goal]{
        \hspace{-5pt}
        \label{fig:scenario_A4}
        \centering
        \includegraphics[width=0.22\textwidth]{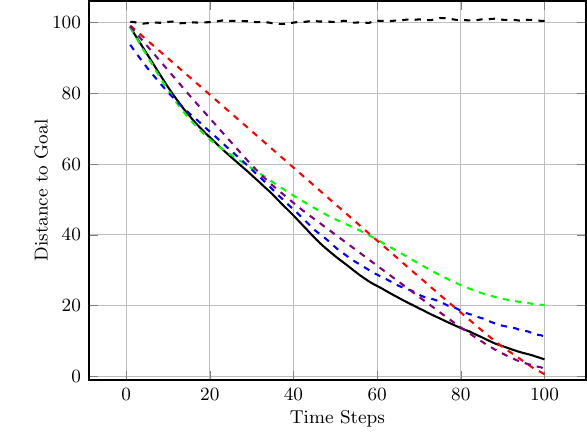}
    }
    \vneg
    \caption{Results for the Obstacle Field scenario. Results for our full approach are marked with a solid black line.  Comparison approaches are marked with dashed lines, with 
    $\beta_2 = 0$ in \color{green}green\color{black}, 
    $\beta_3 = 0$ in \color{red}red\color{black}, 
    $10\%$ noise in \color{blue}blue\color{black}, 
    random in black, and 
    greedy in \color[rgb]{0.3,0,0.7}purple\color{black}. 
    These colors are also used later. [Best viewed in color.]}
    \label{fig:scenario_A_results}
\end{figure*}

\begin{figure*}[h]
    \centering
    \subfigure[Cohesion]{
        \hspace{-12pt}
        \label{fig:scenario_B1}
        \centering
        \includegraphics[width=0.215\textwidth]{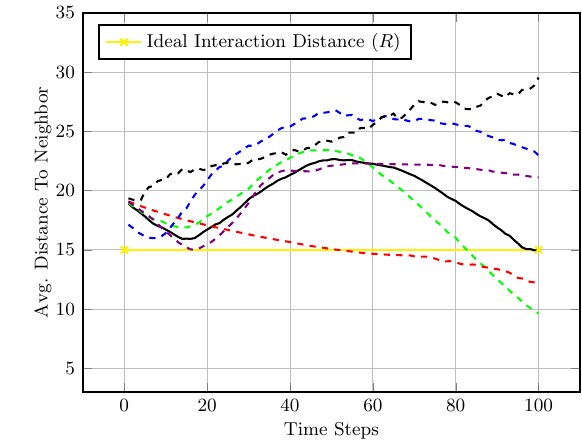}
    }%
    \subfigure[Connectivity]{
        \hspace{-5pt}
        \label{fig:scenario_B2}
        \centering
        \includegraphics[width=0.22\textwidth]{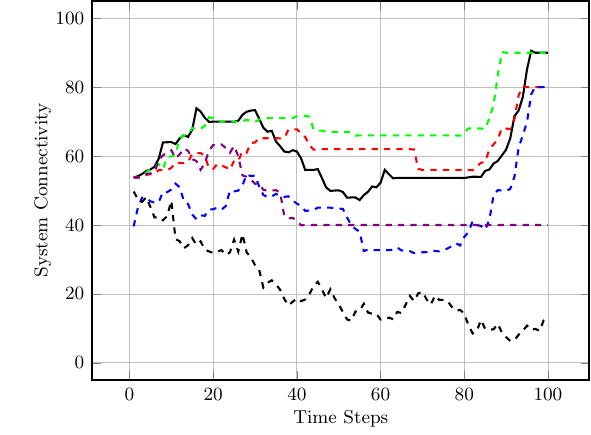}
    }%
    \subfigure[Nearest Obstacle]{
        \hspace{-5pt}
        \label{fig:scenario_B3}
        \centering
        \includegraphics[width=0.215\textwidth]{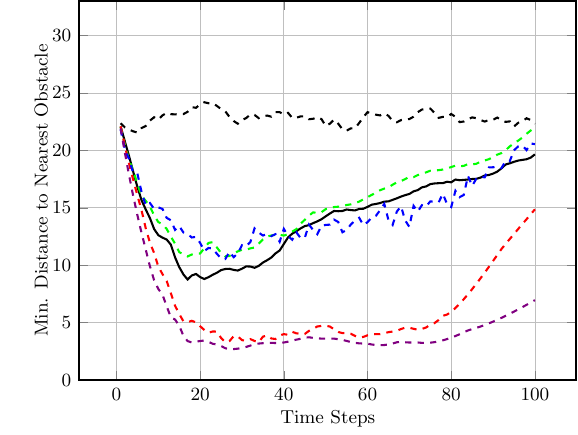}
    }%
    \subfigure[Distance to Goal]{
        \hspace{-5pt}
        \label{fig:scenario_B4}
        \centering
        \includegraphics[width=0.22\textwidth]{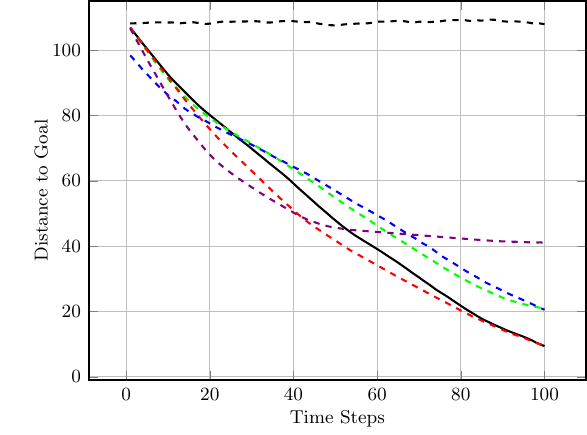}
    }
    \vneg
    \caption{Results for the Area Obstacles scenario. [Best viewed in color.]}
    \label{fig:scenario_B_results}
\end{figure*}

\begin{figure*}[h]
    \centering
    \subfigure[Cohesion]{
        \hspace{-12pt}
        \label{fig:scenario_C1}
        \centering
        \includegraphics[width=0.215\textwidth]{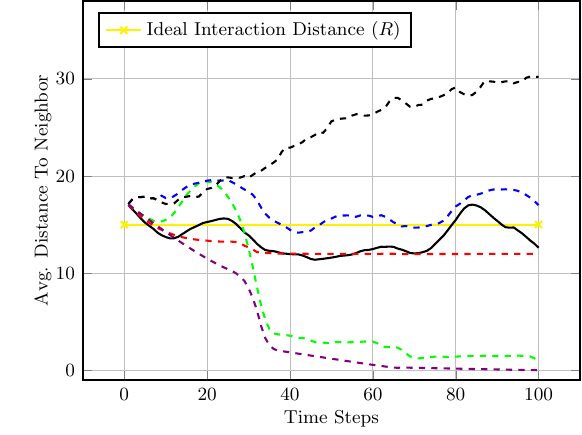}
    }%
    \subfigure[Connectivity]{
        \hspace{-5pt}
        \label{fig:scenario_C2}
        \centering
        \includegraphics[width=0.22\textwidth]{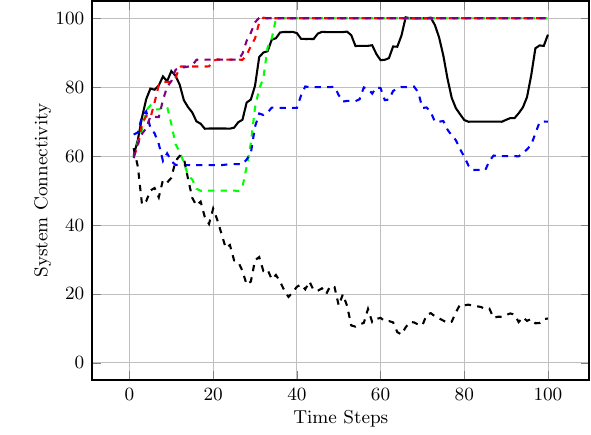}
    }%
    \subfigure[Nearest Obstacle]{
        \hspace{-5pt}
        \label{fig:scenario_C3}
        \centering
        \includegraphics[width=0.215\textwidth]{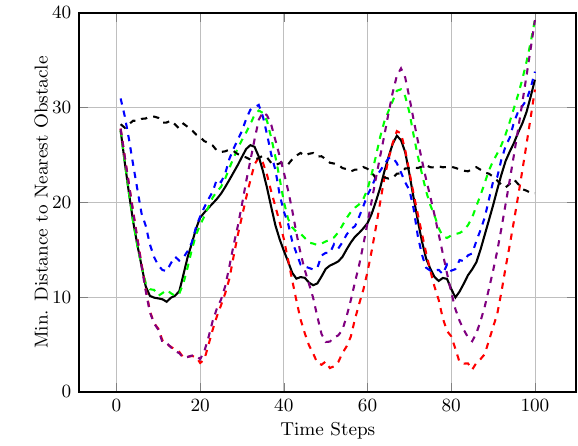}
    }%
    \subfigure[Distance to Goal]{
        \hspace{-5pt}
        \label{fig:scenario_C4}
        \centering
        \includegraphics[width=0.22\textwidth]{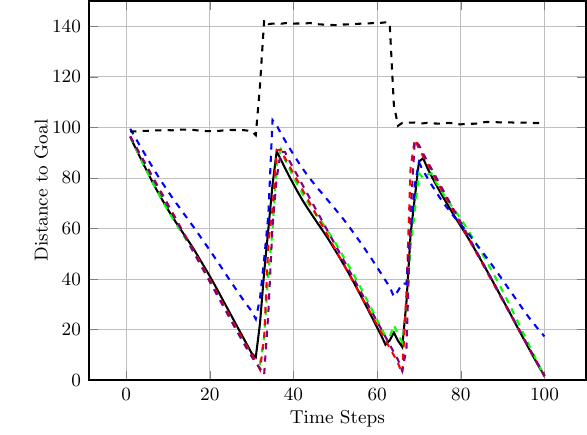}
    }
    \vneg
    \caption{Results for the Dynamic Goal scenario. [Best viewed in color.]}
    \label{fig:scenario_C_results}
\end{figure*}

Figure \ref{fig:scenario_A3} shows the average nearest obstacle to
any robot.
Here, we see that our $\beta_3 = 0$ approach performs poorly for the
first time, with it and the greedy approach moving very near to
obstacles.
Our full approach and our $\beta_2 = 0$ approach perform well, as
they value distance from obstacles.
Finally, Figure \ref{fig:scenario_A4} shows the distance to
the goal over time for each approach.
Only the random approach fails to make progress towards the goal.
Overall, we see that our full approach performs the best, making
progress towards the goal while staying close to the ideal
system cohesion, maintaining
connectivity, and avoiding obstacles.
Each baseline version performs poorly somewhere, with $\beta_2 = 0$
maintaining a poor interaction distance and connectivity and 
$\beta_3 = 0$ navigating much closer to obstacles.

\subsection{Evaluation in Area Obstacles}

We next evaluated the Area Obstacles scenario seen in Figure
\ref{fig:scenario_B}, in which large square
obstacles lie between the multi-robot system and the goal position.
These larger obstacles interfere with robot movement much more than
point obstacles, and so cause larger variation in evaluation.
First, in Figure \ref{fig:scenario_B1} we see that our approach is
the only one capable of ending a simulation with robots distributed
at the ideal cohesion distance.
The alternate versions of our approach result in robots grouped much
too closely together, while the greedy approach ends with robots
spaced too far apart, as robots become stuck navigating through the
large obstacles.
Figure \ref{fig:scenario_B2} shows a similar result when considering
the connectivity status.
All approaches struggle to maintain a fully connected system as the 
robots move through the large obstacles.
Our approach and its variations perform the best, ensuring the multi-robot
system stays together instead of separating around obstacles and
resulting in the highest connectivity scores.
The greedy approach, which attains high connectivity through simple point
obstacles, has much more trouble maintaining connectivity as robots
become stuck without an algorithm to plan navigational behaviors 
around the large obstacles.

In Figure \ref{fig:scenario_B3} we see the distance to the nearest
obstacle as robots progress through the environment.
As with the Obstacle Field scenario, our full approach and the
$\beta_2 = 0$ approach are able to maintain significant distance from
the large obstacles, showing the importance of the obstacle avoidance
term in the cost function.
Again we see the $\beta_3 = 0$ and greedy approaches
result in the lowest average distance to an obstacle, indicating a high
risk of being stuck behind an obstacle or nearing a collision.
Finally, in Figure \ref{fig:scenario_B4} we see that our approach
and its variations are the only approaches able to consistently
reach a goal position through the Area Obstacles.
These approaches, because of their ability to plan actions in the
context of their environments and their neighbors, allow the multi-robot
system to effectively navigate the large obstacles.
The greedy approach, as can be expected by its poor maintenance of
system cohesion and its nearness to obstacles, regularly fails to navigate
the entire system to the goal position.

\subsection{Evaluation with Dynamic Goal}

Finally, we evaluate the scenario seen in Figure \ref{fig:scenario_C}
in which there is a Dynamic Goal, as if the multi-robot system was 
tasked with following waypoints or a human through an environment.
We see similar results to earlier scenarios.
The naive greedy approach causes robots to bunch together, maintaining
cohesion but making the system at risk for collisions.
Our baseline approaches again each struggle in the cost areas that their 
hyperparameters omit.
When $\beta_2 = 0$, we see that the system maintains connectivity
at the expense of allowing robots to become dangerously close together.
When $\beta_3 = 0$, we see that the system maintains a safe distance
between robots but risks collisions with obstacles.

However, our full approach again balances the evaluated metrics.
It is able to maintain an ideal interaction distance among robots, 
keeping them at a safe distance to avoid collisions but still achieving
full connectivity.
This full approach is also able to maintain a safe distance from obstacles
in the environment, while still making progress towards
the waypoints.

\subsection{Robustness to Noise}

In order to understand how our approach would
perform with the noise inherent in physical robot operations
(i.e., uncertainty about the positions of detected neighbors
and obstacles), we conduct evaluations with locations in
the environment modified with $10\%$
sensor noise.
Results with this added noise can be seen in blue
in Figures \ref{fig:scenario_A_results}-\ref{fig:scenario_C_results}.
We can see that our approach does lose performance when noise
is added, but is still able to outperform both of our
baseline methods and the greedy method.
Due to robots being able to play their games before
each action, errors in observation do not build up or persist
and cause long term issues.
Additionally, as our approach is decentralized, errors
from the added noise do not propagate through the system, and
are restricted to individual robots.
Accordingly, our approach is very robust to the sensor and movement
noise encountered in physical multi-robot systems.

\section{Conclusion}
\label{sec:conclusion}

In this paper, we propose a novel game theoretic method for decentralized
and communication-free multi-robot navigation.
Our approach is proven to identify strategies that approach a stable
Nash equilibrium, resulting in optimal behaviors for each
robot.
We show that our approach enables a multi-robot
system to navigate to a goal position,
with changes to hyperparameter values showing the importance of 
terms that encourage both avoiding obstacles and
collisions with neighboring robots and maintaining connectivity,
while remaining robust to errors from sensor noise.

\bibliographystyle{ieeetr}
\bibliography{references}

\end{document}